\documentclass[final,onefignum,onetabnum]{siamonline181217}

\usepackage{lipsum}
\usepackage{amsfonts}
\usepackage{graphicx}
\usepackage{epstopdf}
\usepackage{amsmath,amsfonts,amssymb,bm,mathrsfs}
\usepackage{caption}
\usepackage{graphicx, subfigure}
\usepackage{color}
\usepackage{booktabs}
\usepackage{algorithm}
\usepackage{setspace}
\usepackage{algpseudocode}
\usepackage{epsfig}
\usepackage{float}
\usepackage{cite}
\usepackage{exscale}
\usepackage{relsize}
\usepackage{indentfirst}
\usepackage{multirow}
\usepackage{hyperref}
\graphicspath{{report_image/}}
\ifpdf
  \DeclareGraphicsExtensions{.eps,.pdf,.png,.jpg}
\else
  \DeclareGraphicsExtensions{.eps}
\fi

\usepackage{enumitem}
\setlist[enumerate]{leftmargin=.5in}
\setlist[itemize]{leftmargin=.5in}

\newsiamremark{remark}{Remark}
\newsiamremark{hypothesis}{Hypothesis}
\crefname{hypothesis}{Hypothesis}{Hypotheses}
\newsiamthm{claim}{Claim}

\headers{Mean-field ResNet}{Yiping Lu, Chao Ma, Yulong Lu, Jianfeng Lu, Lexing Ying}

\author{XXX\thanks{XXX 
  (\email{XXX}, \url{XXX}).}
\and XXX\thanks{XXX
  (\email{XXX}, \email{XXX}).}
\and XXX\footnotemark[3]}

\usepackage{amsopn}

\usepackage{parskip}                
\usepackage[thmmarks,amsmath]{ntheorem}
\usepackage{multirow}
\usepackage{amssymb,amsfonts}
\renewcommand{\d}{\mathrm{d}}

\DeclareMathOperator{\dist}{dist}
\renewcommand{\d}{\mathrm{d}}

\renewcommand{\S}{\mathbb{S}}
\newcommand{\s}{\theta}
\newsiamthm{assumption}{Assumption}
\usepackage{dsfont}

\title{A Mean-field Analysis of Deep ResNet and Beyond:\\
Towards Provable Optimization Via Overparameterization From Depth}

\author{Yiping Lu\thanks{Stanford University, ICME.  
  (\email{yplu@stanford.edu}, \url{https://web.stanford.edu/\~ yplu/}).}
\and Chao Ma\thanks{Princeton University 
  (\email{chaom@princeton.edu}).}
\and Yulong Lu\thanks{Department of Mathematics, Duke University,(\email{yulonglu@math.duke.edu})}\and Jianfeng Lu\thanks{Department of Mathematics, Department of Chemistry and Department of Physics, Duke University,(\email{jianfeng@math.duke.edu},\url{https://services.math.duke.edu/\~jianfeng/})} \and Lexing Ying\thanks{Department of Mathematics, Stanford University,(\email{lexing@stanford.edu},\url{https://web.stanford.edu/\~lexing/})}}

\begin{document}

\maketitle
\begin{abstract}
 Training deep neural networks with stochastic gradient descent (SGD) can often achieve zero training loss on real-world tasks although the optimization landscape is known to be highly non-convex. To understand the success of SGD for training deep neural networks, this work presents a mean-field analysis of deep residual networks, based on a line of works that interpret the continuum limit of the deep residual network as an ordinary differential equation when the network capacity tends to infinity. Specifically, we propose a \textbf{new continuum limit} of deep residual networks, which enjoys a good landscape in the sense that \textbf{every local minimizer is global}. 
This characterization enables us to derive the first global convergence result for multilayer neural networks in the mean-field regime. Furthermore, without assuming the convexity of the loss landscape, our proof relies on a zero-loss assumption at the global minimizer that can be achieved when the model shares a universal approximation property. Key to our result is the observation that a deep residual network resembles a shallow network ensemble~\cite{veit2016residual}, \emph{i.e.} a two-layer network. We bound the difference between the shallow network and our ResNet model via the adjoint sensitivity method, which enables us to apply existing mean-field analyses of two-layer networks to deep networks. Furthermore, we propose several novel training schemes based on the new continuous model, including one training procedure that switches the order of the residual blocks and results in strong empirical performance on the benchmark datasets. 
\end{abstract}

\begin{keywords}
  Mean–field Analysis, Deep Learning, Optimization
\end{keywords}

\begin{AMS}
  62H35 65D18 68U10 58C40 58J50
\end{AMS}

\section{Introduction}

Neural networks have become state-of-the-art models in numerous machine learning tasks and strong empirical performance is often achieved by deeper networks. One landmark example is the residual network (ResNet) \cite{he2016deep,he2016identity}, which can be efficiently optimized even at extremely large depth such as 1000 layers. However, there exists a gap between this empirical success and the theoretical understanding: ResNets can be trained to almost zero loss with standard stochastic gradient descent, yet it is known that larger depth leads to increasingly non-convex landscape even the the presence of residual connections \cite{yun2019deep}. While global convergence can be obtained in the so-called ``lazy'' regime e.g. \cite{jacot2018neural,du2018gradient}, such kernel models cannot capture fully-trained neural networks \cite{suzuki2018adaptivity,chizat2019lazy,ghorbani2019limitations}.

In this work, we aim to demonstrate the provable optimization of ResNet beyond the restrictive ``lazy'' regime. To do so, we build upon recent works that connect ordinary differential equation (ODE) models to infinite-depth neural networks~\cite{weinan2017proposal,lu2017beyond,sonoda2017double, haber2017stable,chen2018neural,dupont2019augmented, zhang2018dynamically,thorpe2018deep,sonoda2019transport,lu2019understanding}. Specifically, each residual block of a ResNet can be written as $x_{n+1}=x_n+\Delta t f(x_n,\theta_n)$, which can be seen as the Euler discretization of the ODE $\dot x_t=f(x,t)$. 
This turns training the neural network into solving an optimal control problem~\cite{li2017maximum,weinan2019mean,liu2019deep}, under which backpropagation can be understood as simulating the adjoint equation~\cite{chen2018neural,li2017maximum,pmlr-v80-li18b,zhang2019you,li2020scalable}. 
However, this analogy does not directly provide guarantees of global convergence even in the continuum limit.

\begin{figure}[t]
    \centering
    \includegraphics[width=3in]{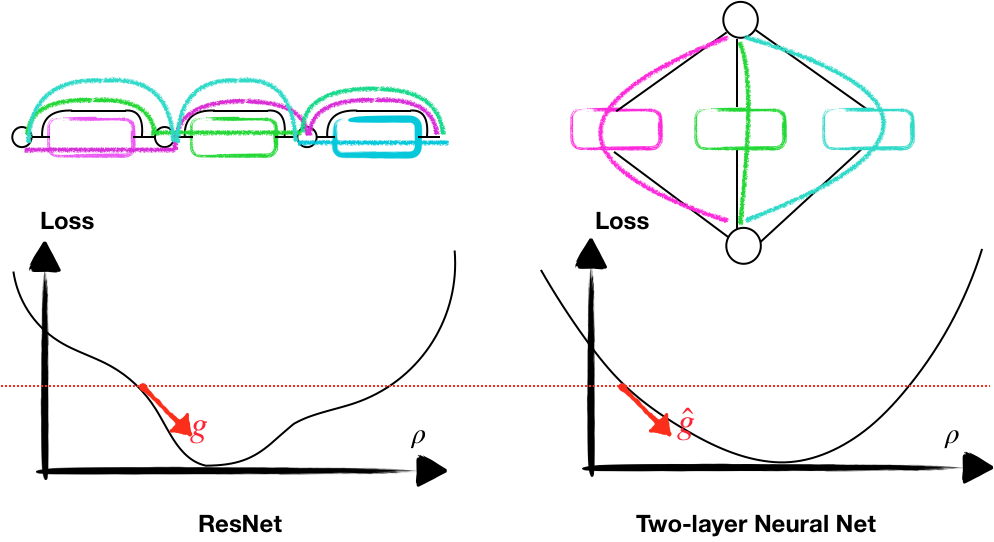}
    \caption{Illustration that ResNet behaves like shallow network ensemble, \emph{i.e.} a two-layer overparameterized neural network. The high-level intuition is to show that the gradient of the two models are at the same scale when the loss are comparable.}
    \label{fig:my_label}
\end{figure}

To address the problem of global convergence, we propose \textbf{a new limiting ODE model} of ResNets. 
Formally, we model deep ResNets via a mean-field ODE model
\begin{equation*}
\dot X_\rho(x,t) = \int_\theta f(X_\rho(x,t),\theta)\rho(\theta,t)d\theta
\end{equation*}
This model considers every residual block $f(\cdot,\theta_i)$ as a particle and optimizes over the empirical distribution of particles $\rho(\theta,t)$, where $\theta$ denotes the weight of the residual block and $t$ denotes the layer index of the residual block. Similar limiting objective function is proposed in \cite{hu2019mean,jabir2019mean,ma2019priori,e2019machine}. \cite{hu2019mean,jabir2019mean} have introduce a further convex condition on the Hamiltonian function which is generally not true for the realistic setting. \cite{ma2019priori} is mainly discussing the statistical property of the objective which is out of the scope of the discussing of this paper.  We consider properties of the loss landscape with respect to the distribution of weights, an approach similar to \cite{bengio2006convex,bach2017breaking}. Inspired by \cite{veit2016residual} that a deep ResNet behaves like an ensemble of shallow models, we compare a deep ResNet with its counterpart two-layer network and show that the gradients of the two models are close to each other. This leads us to conclude that, although the loss landscape may not be convex, every local minimizer is a global one.

\subsection{Contribution}

Our contributions can be summarized as follows:
\begin{itemize}
    \item We derive a new continuous depth limit of deep ResNets. In this new model, each residual block is regarded as a particle and the training dynamics is captured by the gradient flow on the distribution of the particles $\rho$.
    \item We analyze the loss landscape with respect to $\rho$ and show that all local minima have zero loss, which indicates that every local optima is global. This property leads to the conclusion that a full support stationary point of the Wasserstein gradient flow is a global optimum. To the best of our knowledge, this is the \textbf{first global convergence result for multi-layer neural networks in the mean-field regime} without the convexity assumption on the loss landscape.
    \item We propose novel numerical schemes to approximate the mean-field limit of the deep ResNets and demonstrate that they achieves superior empirical results on real-world datasets.
\end{itemize}

\subsection{Related Work}

\paragraph{Mean-Field Limit and Global Convergence.}

Recent works have explored the global convergence of two-layer neural networks by studying suitable scaling limits of the stochastic gradient descent of two-layer neural network when the width is sent to infinity and the second layer scaled by one over the width of the neural network~\cite{nitanda2017stochastic, mei2018mean,rotskoff2018neural,chizat2018global,sirignano2019mean}. Though global convergence can be obtained under certain conditions for two-layer networks, it is highly nontrivial to extend this framework to multi-layer neural networks: recent attempts \cite{araujo2019mean,sirignano2019mean,nguyen2019mean,fang2019convex} do not address realistic neural architectures directly or provide conditions for global convergence. 
Parallel to the mean-field regime, \cite{jacot2018neural,du2018gradient,allen2018convergence,zou2018stochastic,oymak2019towards} provided global convergence results for multi-layer networks in the so-called "lazy" or kernel regime. However, this description of deep neural networks is rather limited: the scaling of initialization forces the distance traveled by each parameter to vanish asymptotically \cite{chizat2019lazy}, and thus training becomes equivalent to kernel regression with respect to \emph{neural tangent kernel}~\cite{arora2019exact,jacot2018neural}.
On the other hand, it is well-known that properly trained neural networks can outperform kernel models in learning various target functions \cite{wei2019regularization,suzuki2018adaptivity,ghorbani2019limitations,ba2020generalization,allen2019can}. In contrast, the mean-field regime considered in this work does not reduce training into kernel regression; in other words, the mean-field setting allows neurons to travel further and learn adaptive features.

\paragraph{Landscape of ResNets.}  \cite{li2017convergence,liu2019towards} provided convergence results of gradient descent on two-layer residual neural networks and showed that the global minimum is unique. 
In parallel, \cite{shamir2018resnets,kawaguchi2019depth} showed that when the network consists of one residual block the gradient descent solution is provably better than a linear classifier. However, recent work also pointed out that these positive results may not hold true for deep ResNets composed of multiple residual blocks. Regarding deeper models, \cite{hardt2016identity,bartlett2019gradient,wu2019global} proved the global convergence of the gradient descent for training deep {\em linear} ResNets. Yet it is known that even mild nonlinear activation functions can destroy these good landscape properties \cite{yun2018small}. In addition, \cite{bartlett2018representing} considered a
ResNet model with compositions of close-to-identity functions, and provided convergence result regarding the Fr\'echet gradient. However, \cite{bartlett2018representing} also pointed out that such conclusion may no longer hold for a realistic ResNet model. Our paper fills this gap by introducing a new continuous model and providing conditions for the global convergence beyond the previously considered kernel regime  \cite{du2018gradient,zhang2019training,allen2018convergence,zhang2019training}.

\subsection{Notations and Preliminaries}

\paragraph{Notations.} Let $\delta(\cdot)$ denote the Dirac mass and $1_\Omega$ be the indicator function on $\Omega$. We denote by $\mathcal{P}^2$ the set of probability measures endowed with the Wasserstein-2 distance (see below for definition). Let $\mu$ be the population distribution of the input data and the induced norm by $\|f\|_\mu=\sqrt{\mathbb{E}_{x\sim\mu} [f(x)^\top f(x)]}$.

\paragraph{Fr\'echet Derivative.} We extend the notion of the gradient to infinite dimensional space. For a functional $f:X\rightarrow \mathbb{R}$ defined on a Banach space $X$, the Fr\'echet derivative is an element in the dual space $df\in X^*$ that satisfies
\begin{align*}
\lim_{\delta\in X,\delta\rightarrow 0} \frac{f(x+\delta)-f(x)-df(\delta)}{\|\delta\|}=0, \quad \text{for all} \; x\in X.
\end{align*}
In this paper, $\frac{\delta f}{\delta X}$ is used to denote the Fr\'echet derivative.

\paragraph{Wasserstein Space.} The Wasserstein-$2$ distance between two probability measures $\mu,\nu\in\mathcal{P}(\mathbb{R}^d)$ is defined as
$$
W_2(\mu,\nu):=\left(\inf_{\gamma\in\mathcal{T}(\mu,\nu)}\int|y-x|^2 d\gamma(x,y)\right)^{1/2}.
$$
Here $\mathcal{T}(\mu,\nu)$ denotes the set of all couplings between $\mu$ and $\nu$, i.e., all probability measures $\gamma\in\mathcal{P}(\mathbb{R}^d\times\mathbb{R}^d)$ with marginals $\mu$ on the first factor and $\nu$ on the second.

\paragraph{Bounded Lipschitz norm.} We say that a sequence of measures $\mu_n \in \mathcal{M}(\mathbb{R}^d)$ \emph{weakly} (or \emph{narrowly}) converges to $\mu$ if, for all continuous and bounded function $\varphi: \mathbb{R}^d \to \mathbb{R}$ it holds $\int \varphi \d \mu_n \to \int \varphi \d \mu$. For sequences which are bounded in total variation norm, this is equivalent to the convergence in Bounded Lipschitz norm. The latter is defined, for $\mu \in \mathcal{M}(\mathbb{R}^d)$, as 
\begin{equation}\label{eq:BL}
\Vert \mu \Vert_{\text{BL}} := \sup \left\{  \int \varphi \, \d \mu\; ;\; \varphi: \mathbb{R}^d \to \mathbb{R}, \; \text{Lip}(\varphi)\leq 1,\; \Vert \varphi \Vert_\infty \leq 1 \right\}
\end{equation}
where $\text{Lip}(\varphi)$ is the smallest Lipschitz constant of $\varphi$ and $\Vert \cdot \Vert_\infty$ the supremum norm. 

\section{Limiting Model}
Following the observation that each residual block of a ResNet  $u_{n+1}=u_n+\Delta t f(u_n,\theta_n)$ can be considered as one step of the forward Euler approximation of the ODE $u_t=f(u,t)$~\cite{weinan2017proposal,lu2017beyond,sonoda2017double,haber2017stable}, a series of recent papers \cite{zhang2018dynamically,zhang2019you,chen2018neural,li2020scalable,li2017maximum,li2020scalable} analyzed the deep neural networks in the continuous limit. \cite{thorpe2018deep} proved the Gamma-convergence of ResNets in the asymptotic limit. However, there are two points of that approach that require further investigation. First, \cite{thorpe2018deep} introduced a regularization term $n\sum_{i=1}^n\|\theta_i-\theta_{i-1}\|^2$, where $n$ is the depth of the network. This regularization becomes stronger as the network gets deeper, which implies a more constrained space of functions that the network can represent.

Second, while the Gamma-convergence result is concerned with the convergence of the global minima of a sequence of energy functionals, it gives rather little information about the landscape of the limiting functional, which can be quite complicated for non-convex objective functions. Later work \cite{avelin2019neural} proved that  stochastic gradient descent of a deep ResNet with constant weight across layers converges to the gradient flow of loss using the ODE model. However, letting the weights of the ResNet be the same across all layers weakens the approximation power and makes optimization landscape more complicated. To address the reason behind the global convergence of the gradient flow, in this section, we propose a new continuous limiting model of the deep residual network.

\subsection{A New Continuous Model}

The goal is to minimize the $l_2$ loss function
\begin{equation}
E(\rho)=\mathbb{E}_{x\sim\mu} \Big[\frac{1}{2}\left(\left<w_1,X_\rho(x,1)\right>-y(x) \right)^2 \Big].  
\label{con:objective}
\end{equation}
over parameter distributions $\rho(\theta,t)$ for $\theta$ in a compact set $\Omega$ and $t\in[0,1]$. Here $X_\rho(x,t)$ is the solution of the ODE
\begin{equation}
\dot X_\rho(x,t) = \int_\theta f(X_\rho(x,t),\theta)\rho(\theta,t)d\theta,X_\rho(x,0) = \left<w_2,x\right> 
\label{con:model}
\end{equation}
The ODE \eqref{con:model} is understood in the integrated sense, \textit{i.e.}, for fixed distribution $\rho(\cdot,\cdot)$ and input $x \in \mathbb{R}^{d_1}$, the solution path $X_{\rho}(x, t), t \in [0, 1]$ satisfies
\begin{equation*}
X_{\rho}(x, t) = X_{\rho}(x, 0)  + \int_0^t \int_{\Omega} f(X_{\rho}(x, s), \theta) \rho(\theta, s) d\theta d s.
\end{equation*}
Here $y(x)=\mathbb{E}[y|x]\in \mathbb{R}$ is the function to be estimated. The parameter $w_2\in\mathbb{R}^{d_1\times d_2}$ represents the first convolution layer in the ResNet \cite{he2016deep,he2016identity}, which extracts feature before sending them to the residual blocks.  To simplify the analysis, we let $w_2$ to a predefined linear transformation (\emph{i.e.} not training the first layer parameters $w_2$) with the technical assumption that $\min\{\sigma(w_2)\}\ge \sigma_1$ and $\max\{\sigma(w_2)\}\le \sigma_2$, where $\sigma(w_2)$ denotes the set of singular values. We remark that this assumption is not unrealistic, for example \cite{oyallon2017scaling} let $w_2$ be a predefined wavelet transform and still achieved the state-of-the-art result on several benchmark datasets. Here $f(\cdot,\theta)$ is the residual block with parameter $\theta$ that aims to learn a feature transformation from $\mathbb{R}^{d_2}$ to $\mathbb{R}^{d_2}$. For simplicity, we assume that the residual block is a two layer neural network, thus $f(x,\theta)=\sigma(\theta x), \theta\in\Omega \subset \mathbb{R}^{d_2\times d_2}$ and $\sigma:\mathbb{R}\rightarrow\mathbb{R}$ is an activation function, such as sigmoid and relu. Note that in our notation $\sigma(\theta x)$ the activation function $\sigma$ is applied separately to each component of the vector. 

Finally, $w_1\in\mathbb{R}^{d_2\times 1}$ is a pooling operator that transfers the final feature $X_\rho(x,1)$ to the classification result and an $l_2$ loss function is used for example. We also assume that $w_1$ is a predefined linear transform with satisfies $\|w_1\|_2=1$, which can be easily achieved via an operator used in realistic architecture such as the global average pooling~\cite{lin2013network}. Before starting the analysis, we first list the necessary regularity assumptions.

\begin{assumption}
\label{assump:1}
\begin{enumerate}
\item (Boundedness of data and target distribution) The input data $x$ lies $\mu$-almost surely
  in a compact ball, \emph{i.e.} $\|x\|\le R_1$ for some constant $R_1>0$. At the same time the
  target function is also bounded $\|y(\cdot)\|_{\infty}\le R_2$ for some constant $R_2>0$.
\item (Lipschitz continuity of distribution with respect to depth) There exists a constant
  $C_\rho$ such that
  \[
  \|\rho(\cdot,t_1)-\rho(\cdot,t_2)\|_{BL}\le C_\rho |t_1-t_2|
  \]
  for all $t_1,t_2\in[0,1]$. 
\item The kernel $k(x_1, x_2) := g(x_1,x_2)=\sigma(x_1^\top x_2)$ is a universal kernel
  \cite{micchelli2006universal}, i.e.  the span of $\{k(x,\cdot):x\in \mathbb{R}^{d_2}\}$ is dense
  in $L^2$.
    \item (Locally Lipschitz derivative with sub-linear growth~\cite{chizat2018global}) There exists a family $\{Q_r\}_{r>0}$ of nested nonempty closed convex subsets of $\Omega$ that satisfies:
        \begin{itemize}
        \item $\{u\in\Omega\mid\dist(u,Q_r)\le r'\}\subset Q_{r+r'}$ for all $r,r'>0$.
        \item There exist constants $C_1,C_2>0$ such that
        $$\sup_{\theta\in Q_r, x} \|\nabla_x f(x,\theta)\|\le C_1+C_2 r$$ 
        holds for all $r>0$.  Also the gradient of $f(x,\theta)$ with respect to $x$ is a Lipschitz function with Lipschitz constant $L_r>0$.
        \item For each $r$, the gradient respect to the parameter $\theta$ is also bounded
        $$
        \sup_{\|x\|\le R_1,\theta\in Q_r}\|\nabla_\theta f(x,\theta)\|\le C_{3,r}
        $$
        for some constant $C_{3,r}$.
    \end{itemize}
\end{enumerate}
\end{assumption}

\begin{remark}
Let us elaborate on these assumptions in the neural network setting. For Assumption 1.4, $k(x_1, x_2) := g(x_1,x_2)=\sigma(x_1^\top x_2)$ is a universal kernel holds for the sigmoid and ReLU activation function.  The local regularity Assumption 1.5 concerning function $f(x,\theta)$ can easily be satisfied, for $\nabla_\theta \sigma(\theta^\top x)=\sigma'(\theta^\top x)x$ and $\nabla_x\sigma(\theta^\top x)=\sigma'(\theta^\top x)\theta$. Hence, in order to satisfy the local regularity condition, one possible solution is that we utlize a Lipschitz gradient activation function and set the local set $Q_r$ to be a ball with radius $r$ centered at origin.
\end{remark}

Under these assumptions, we can establish the existence, uniqueness, stability, and well-posedness of our forward model.
\begin{theorem} [Well-posedness of the Forward Model]\label{thm:wellpose}
Under Assumption 1 and we further assume that there exists a constant $r>0$ such that $\mu$ is
concentrated on one of the nested sets $Q_r$. Then, the ODE in~\eqref{con:model} has a unique solution in $t\in[0,1]$ for any initial condition
$x\in\mathbb{R}^{d_1}$. Moreover, for any pair of distributions $\rho_1$ and $\rho_2$, there exists
a constant $C$ such that
\begin{equation}
    \|X_{\rho_1}(x,1)-X_{\rho_2}(x,1)\|<C W_2(\rho_1,\rho_2),
\end{equation}
where $W_2(\rho_1,\rho_2)$ is the 2-Wasserstein distance between $\rho_1$ and $\rho_2$. 
\end{theorem}
\begin{proof}
We first show the existence and uniqueness of $X_\rho(x,t)$. From now on, let
\begin{equation}
    F_\rho(X,t) = \int_\theta f(X,t)\rho(\theta,t)d\theta.
\end{equation}
Then, the ODE~\eqref{con:model} becomes
\begin{equation}\label{eq:ode}
    \dot X_\rho(x,t) = F_\rho(X_\rho(x,t),t),
\end{equation}
and by the condition of the theorem and assumption~\ref{assump:1} we have
\begin{equation}
    \|F_\rho(X,t)\|\leq C_f^r\left|\int_\theta \rho(\theta,t)d\theta\right|< C_f^rC_\rho.
\end{equation}

This is because, for the continuous function $f(x,\theta)$ is now defined on the domain for which
$\theta$ lies in a compact set $Q_r$ and $\|x\|<R_1$, which leads to
an upper bound $C_f^r$ such that $\sup_{\|x\|<R}f(x,\theta)<C_f^r$ holds for all $\theta\in
Q_r$. The notation $C_f^r$ will continuously used in the following section.

Hence, $F_\rho(X_\rho,t)$ is bounded. On the other hand, $F_\rho(X,t)$ is integrable with respect to
$t$ and Lipschitz continuous with respect to $X$ in any bounded region (by 2 of
assumption~\ref{assump:1}). Therefore, consider the region $[X_0-C_f^rC_\rho,
  X_0+C_f^rC_\rho]\times[0,1]$, where $X_0=X_\rho(x,0)$. By the existence and uniqueness theorem of
ODE (the Picard–Lindel\"of theorem), the solution of~\eqref{eq:ode} initialized from $X_0$ exists
and is unique on $[0,1]$.

Next, we show the continuity of $X_\rho(x,t)$ with respect to $\rho$. Letting
$\Delta(x,t)=\|X_{\rho_1}(x,t)-X_{\rho_2}(x,t)\|$, we have
\begin{align}
\Delta(x,t) &= \left\| \int_0^t \dot X_{\rho_1}(x,s)-\dot X_{\rho_2}(x,s) ds \right\| \nonumber \\
  &=\left\| \int_0^t F_{\rho_1}(X_{\rho_1},s)-F_{\rho_1}(X_{\rho_2},s)ds +\int_0^t F_{\rho_1}(X_{\rho_2},s)-F_{\rho_2}(X_{\rho_2},s)ds\right\| \nonumber\\
  &\leq \int_0^t \|F_{\rho_1}(X_{\rho_1},s)-F_{\rho_1}(X_{\rho_2},s)\|ds + 
  \left\| \int_0^t F_{\rho_1}(X_{\rho_2},s)-F_{\rho_2}(X_{\rho_2},s)ds \right\|. \label{eq:Delta}
\end{align}
Let $C_m=\max\{C_{\rho_1},C_{\rho_2}\}$. For the first term in~\eqref{eq:Delta}, since both
$X_{\rho_1}$ and $X_{\rho_2}$ are controlled by $X_0+C_f^rC_m$, by 2 of Assumption~\ref{assump:1}
 we have the following Lipschitz condition for
\begin{equation}
\|F_{\rho_1}(X_{\rho_1},s)-F_{\rho_1}(X_{\rho_2},s)\|\leq (C_1+C_2X_0+C_2C_f^rC_m)C_m\Delta(x,s).
\end{equation}
For the second term of~\eqref{eq:Delta}, we have
\begin{align}
\left\| \int_0^t F_{\rho_1}(X_{\rho_2},s)-F_{\rho_2}(X_{\rho_2},s)ds \right\| &= \left\| \int_0^t \int_\theta f(X_{\rho_2},\theta)(\rho_1(\theta,s)-\rho_2(\theta,s))d\theta ds \right\|.
\end{align}
Since $X_{\rho_2}$ is $C_f^rC_m$-Lipschitz continuous with respect to $t$ and also bounded by
$X_0+C_f^rC_m$, we have $f(X_{\rho_2},\theta)$ is $(C_1+C_2X_0+C_2C_f^rC_m)C_f^rC_m$-Lipschitz
continuous w.r.t $t$. On the other hand, still by Assumption~\ref{assump:1}, 
$f(X,\theta)$ is $C_{3,r}$-Lipschitz with respect to $\theta$. As a result, the function
$f(X_{\rho_2},\theta)$ is $C$-Lipschitz continuous on $(t,\theta)$ with
$C=(C_1+C_2X_0+C_2C_f^rC_m)C_f^rC_m+C_{3,r}$, which implies

\begin{equation}
\left\| \int_0^t \int_\theta f(X_{\rho_2},\theta)(\rho_1(\theta,s)-\rho_2(\theta,s))d\theta ds \right\|\leq CW_2(\rho_1,\rho_2). 
\end{equation}

Finally, by defining 
\begin{equation}
\hat C = \max\{(C_1+C_2X_0+C_2C_f^rC_m)C_m, C\},
\end{equation}
we have by~\eqref{eq:Delta} 
\begin{equation}
    \Delta(x,t)\leq \int_0^t \hat C\Delta(x,t) + \hat CW_2(\rho_1,\rho_2).
\end{equation}
Applying the Gronwall's inequality gives
\begin{equation}
    \Delta(x,t)\leq \hat Ce^{\hat Ct}W_2(\rho_1,\rho_2),
\end{equation}
and specifically for $t=1$ we have
\begin{equation}
    \|X_{\rho_1}(x,1)-X_{\rho_2}(x,1)\|\leq \hat Ce^{\hat C}W_2(\rho_1,\rho_2). 
\end{equation}

\end{proof}

\subsection{Deep Residual Network Behaves Like an Ensemble Of Shallow Models}
\label{ensemble}

In this section, we briefly explain the intuition behind our analysis, \emph{i.e.} deep residual network can be approximated by a two-layer neural network. \cite{veit2016residual} introduced an unraveled view of the ResNets and showed that deep ResNets behave like ensembles of shallow models. First, we offer a formal derivation to reveal how to make connection between a deep ResNet and a two-layer neural network. The first residual block is formulated as 
$$
X^1 = X^0 + \frac{1}{L} \int_{\theta^0} \sigma(\theta^0X^0)\rho^0(\theta^0)d\theta^0.
$$
By Taylor expansion, the second layer output is given by
\begin{align*}
    &X^2 = X^1 +\frac{1}{L} \int_{\theta^1} \sigma(\theta^1X^1)\rho^1(\theta^1)d\theta^1\\
    &=X^0 + \frac{1}{L} \int_{\theta^0} \sigma(\theta^0X^0)\rho^0(\theta^0)d\theta^0 \\&
+\int_{\theta^1} \sigma(\theta^1(X^0 + \frac{1}{L} \int_{\theta^0} \sigma(\theta^0X^0)\rho^0(\theta^0)d\theta^0))\rho^1(\theta^1)d\theta^1\\
&=X^0 + \frac{1}{L} \int_{\theta^0} \sigma(\theta^0X^0)\rho^0(\theta^0)d\theta^0\\& + X^0 + \frac{1}{L} \int_{\theta^1} \sigma(\theta^1X^0)\rho^1(\theta^1)d\theta^1\\
&+\frac{1}{L^2}\int_{\theta_1}\nabla\sigma(\theta^1X^0)\theta^1(\int_{\theta^0}\sigma(\theta^0X^0)\rho^0(\theta^0)d\theta^0)\rho^1(\theta^1)d\theta^1\\&+h.o.t.
\end{align*}
Iterating this expansion gives rise to 
\begin{align*}
&X^L\approx X^0+\frac{1}{L} \sum_{a=0}^{L-1} \int \sigma(\theta X^0)\rho^a(\theta)d\theta\\&+\frac{1}{L^2}\sum_{b>a} \int\int \nabla\sigma(\theta^bX^0)\theta^b\sigma(\theta^aX^0)\rho^b(\theta^b)\rho^a(\theta^a)d\theta^b\theta^a\\&+h.o.t.
\end{align*}
Here we only keep the terms that are at most quadratic in $\rho$. A similar derivation shows that at order $k$ in $\rho$ there are $L \choose k$ terms with coefficient $\frac{1}{L^k}$ each. This implies that the $k$-th order term  in $\rho$ decays as $O(\frac{1}{k!})$, suggesting that one can approximate a deep network by the keeping a few leading orders. 
\section{Landscape Analysis of the Mean-Field Model}

In the following, we show that the landscape of a deep residual network enjoys the extraordinary property that any local optima is global, by comparing the gradient of deep residual network with the mean-field model of two-layer neural network~\cite{mei2018mean,chizat2019lazy,nitanda2017stochastic}. 
To estimate the accuracy of the first order approximation (\emph{i.e.} linearization), we apply the adjoint sensitivity analysis~\cite{boltyanskiy1962mathematical} and show that the difference between the gradient of two models can be bounded via the stability constant of the backward adjoint equation. More precisely,
the goal is to show the backward adjoint equation will only affect the gradient in a bounded constant.

\subsection{Gradient via the Adjoint Sensitivity Method}

\paragraph{Adjoint Equation.} 

To optimize the objective \eqref{con:objective}, we calculate the gradient $\frac{\delta E}{\delta \rho}$ via the \emph{adjoint sensitivity method} \cite{boltyanskiy1962mathematical}. 
To derive the adjoint equation, we first view our generative models where $\rho$ is treated as a parameter as
\begin{equation}
    \dot{X}(x, t) = F(X(x, t); \rho),
\end{equation}
with 
\begin{equation}
    F(X(x, t); \rho) = \int f(X(x, t); \theta) \rho(\theta, t) \,\mathrm{d} \theta.
\end{equation}

The loss function can be written as 
\begin{equation}
    \mathbb{E}_{x\sim\mu} E(x; \rho) := \mathbb{E}_{x\sim\mu} \frac{1}{2}\bigl\lvert \langle w_1, X_{\rho}(x, 1) \rangle - y(x)\bigr\rvert^2
\end{equation}
Define 
\begin{equation}
    p_{\rho}(x, 1) := \frac{\partial E(x; \rho)}{\partial X_{\rho}(x, 1)} = \bigl( \langle w_1, X_{\rho}(x, 1)\rangle - y(x) \bigr) w_1  
\end{equation}
The derivative of $X(x, 1)$ with respect to $X(x, s)$, denoted by the Jacobian $J_{\rho}(x, s)$, satisfies at any previous time $s \leq 1$ the adjoint equation of the ODE
\begin{equation}
    \dot{J}_{\rho}(x, s) = - J_{\rho}(x, s) \nabla_X F(X_{\rho}(x, s); \rho). 
\end{equation}
Next, the perturbation of $E$ by $\rho$ is given by chain rule as 
\begin{equation}
    \begin{aligned}
    \frac{\delta E}{\delta \rho(s)} & = \frac{\partial E}{\partial  X_{\rho}(X, 1)} \frac{\delta X_{\rho}(x, 1)}{\delta \rho(s)} & \\
    & = \frac{\partial E}{\partial  X_{\rho}(X, 1)} J_{\rho}(x, s) \frac{\delta F(X_{\rho}(x, s); \rho)}{\delta \rho(s)} \\
        & = p_{\rho}(x, s)\, f(X_{\rho}(x, has), \cdot), 
    \end{aligned}
\end{equation}
where $p_{\rho}(x, s)$ (the derivative of $E(x; \rho)$ with respect to $X_{\rho}(x, s)$) satisfies the adjoint equation
\begin{align*}
      &\dot p_\rho(x,t)=-\delta_XH_\rho(p_\rho,x,t)\\
      &=-p_\rho(x,t)\int \nabla_X f(X_\rho(x,t),\theta)\rho(\theta,t)d\theta,
\end{align*}
which represents the gradient as a second backwards-in-time augmented ODE. Here the Hamiltonian is defined as $H_\rho(p,x,t)=p(x,t) \cdot \int f(x,\theta)\rho(\theta,t)d\theta$.

Utilizing the adjoint equation, we can characterize the gradient of our model with respect to the distribution $\rho$. More precisely, we may characterize the variation of the loss function with respect to the distribution as the following theorem.
\begin{theorem}[Gradient of the parameter]\label{thm:grad} For $\rho\in \mathcal{P}^2$ let
$$ \frac{\delta E}{\delta \rho}(\theta,t) = \mathbb{E}_{x\sim\mu} f(X_{\rho}(x,t),\theta))p_\rho(x,t).$$ Then for every $\nu\in\mathcal{P}^2$,  we have
$$
E(\rho+\lambda (\nu-\rho)) = E(\rho)+\lambda\left<\frac{\delta E}{\delta \rho}, (\nu-\rho)\right>+o(\lambda)
$$
for the convex combination $(1-\lambda) \rho+\lambda \nu \in\mathcal{P}^2$ with $\lambda \in [0,1]$.
\end{theorem} 
\begin{proof}
To simplify the notation, we use $\hat \rho_\lambda=\rho+\lambda(\rho-\nu)$, From Theorem 1 (the
well-poseness of the model), we know that the function $f(\lambda)=E(\hat\rho_\lambda)-E(\rho)$ is a
continuous function with $f(0)=0$ and thus

\begin{align*}
E(\hat \rho_\lambda)-E(\rho)&=\mathbb{E}_{x\sim\mu}|\left<w_1,X_{\hat
  \rho_\lambda}(x,1)\right>-y(x)|^2-\mathbb{E}_{x\sim\mu}|\left<w_1,X_{\rho}(x,1)\right>-y(x)|^2\\&=
\mathbb{E}_{x\sim\mu}(\left<w_1,X_{\rho}\right>-y(x))(X_{\hat
  \rho_\lambda}(x,1)-X_{\rho}(x,1))+O(X_{\hat \rho_\lambda}(x,1)-X_{\rho}(x,1))
\end{align*}

Now we bound $X_{\hat\rho_\lambda}(x,1)-X_\rho(x,1)$. First, notice that the adjoint equation is a
linear equation:
\begin{align*}
      \dot p_\rho(x,t)=-\delta_X H_\rho(p_\rho,x,t)=-p_\rho(x,t)\int \nabla_X f(X_\rho(x,t),\theta)\rho(\theta,t)d\theta
\end{align*}
with solution
\[
p(x,t)=p(x,1)\exp(\int_t^1 \int \nabla_X f(X_\rho(x,t),\theta)\rho(\theta,t)d\theta dt).
\]
Next, we bound $\Delta(x,t)=\|X_{\hat\rho_\lambda}(x,t)-X_{\rho}(x,t)-\lambda \int_t
\int_\theta(\rho(x,\theta)-\nu(x,\theta))p_\rho(x,t)\|$ in order to show that
$\Delta(x,t)=o(\lambda)$. The way to estimate the difference is to utilize the Duhamel's principle.

\begin{align*}
    &\frac{d}{dt}\left[e^{-\int_0^t \int \nabla_X f(X_\rho(x,t),\theta)\rho(\theta,s)d\theta ds}(X_{\hat\rho_\lambda}(x,s)-X_{\rho}(x,s))\right]\\&= e^{-\int_0^t \int \nabla_X f(X_\rho(x,t),\theta)\rho(\theta,s)d\theta ds}\left[\dot X_{\hat\rho_\lambda}(x,s)-\dot X_{\rho}(x,s)-\int_\theta \nabla_X f(X_\rho(x,t),\theta)\rho(\theta,t)d\theta  (X_{\hat\rho_\lambda}(x,s)- X_{\rho}(x,s))\right]
\end{align*}

At the same time we have 
\begin{align*}
  \dot X_{\hat\rho_\lambda}(x,s)-\dot X_{\rho}(x,s)&=F_{\rho}(X_{\hat\rho_\lambda},s)-F_{\rho}(X_{\rho},s)+F_{\hat\rho_\lambda}(X_{\hat\rho_\lambda},s)-F_{\rho}(X_{\hat\rho_\lambda},s)\\
  &= \left(\int_\theta \nabla_X f(X_\rho(x,t),\theta)\rho(\theta,t)d\theta\right)( X_{\hat\rho_\lambda}(x,s)- X_{\rho}(x,s))+o(\lambda) \\&+ \lambda \int_\theta f(X_{\hat \rho_\lambda}(x,s),\theta)(\rho-\nu)(\theta,s)d\theta 
  \\&=\left(\int_\theta \nabla_X f(X_\rho(x,t),\theta)\rho(\theta,t)d\theta\right)( X_{\hat\rho_\lambda}(x,s)- X_{\rho}(x,s))+o(\lambda)  \\&+ \lambda\left( \int_\theta \nabla_X f(X_{ \rho}(x,s),\theta)(\rho-\nu)(\theta,s)d\theta \right)(X_{\hat\rho_\lambda}(x,s)- X_{\rho}(x,s))+ o(\lambda) 
  \\&+ \lambda \int_\theta f(X_{ \rho}(x,s),\theta)(\rho-\nu)(\theta,s)d\theta
  \\&=\left(\int_\theta \nabla_X f(X_\rho(x,t),\theta)\rho(\theta,t)d\theta\right)( X_{\hat\rho_\lambda}(x,s)- X_{\rho}(x,s)) \\&+ \lambda \int_\theta f(X_{ \rho}(x,s),\theta)(\rho-\nu)(\theta,s)d\theta + o(\lambda).
\end{align*}
Here $F_\rho(X,t) = \int_\theta f(X,t)\rho(\theta,t)d\theta,$ and the last equality holds because
$\|X_{\hat\rho_\lambda}(x,s)- X_{\rho}(x,s)\|\le \hat Ce^{\hat C} d(\rho_1,\rho_2)=O(\lambda)$. This
leads us to
\begin{align*}
  &\frac{d}{dt}\left[e^{-\int_0^t \int \nabla_X f(X_\rho(x,t),\theta)\rho(\theta,s)d\theta ds}(X_{\hat\rho_\lambda}(x,s)-X_{\rho}(x,s))\right]\\&= e^{-\int_0^t \int \nabla_X f(X_\rho(x,t),\theta)\rho(\theta,s)d\theta ds}\biggl[\dot X_{\hat\rho_\lambda}(x,s)-\dot X_{\rho}(x,s)\\
  &\hspace{10em} -\int_\theta \nabla_X f(X_\rho(x,t),\theta)\rho(\theta,t)d\theta  (X_{\hat\rho_\lambda}(x,s)- X_{\rho}(x,s))\biggr]
  \\&=e^{-\int_0^t \int \nabla_X f(X_\rho(x,t),\theta)\rho(\theta,s)d\theta ds}\left[\lambda \int_\theta f(X_{ \rho}(x,s),\theta)+o(\lambda)\right].
\end{align*}
Thus
\begin{align*}
  X_{\hat\rho_\lambda}(x,1)-X_{\rho}(x,1) = \int_0^1\int_\theta e^{\int_t^1 \int \nabla_X f(X_\rho(x,s),\theta)\rho(\theta,s)d\theta ds}  f(X_{ \rho}(x,s),\theta)(\rho-\nu)(\theta,t)d\theta dt+o(\lambda).
\end{align*}

Combining with the definition of the adjoint equation $p(x,t)=p(x,1)e^{\int_t^1 \int \nabla_X
  f(X_\rho(x,t),\theta)\rho(\theta,t)d\theta dt}$ and $ p_{\rho}(x, 1) := \frac{\partial E(x;
  \rho)}{\partial X_{\rho}(x, 1)} = \bigl( \langle w_1, X_{\rho}(x, 1)\rangle - y(x) \bigr) w_1$, we
have
\begin{equation*}
E(\rho+\lambda (\rho-\nu)) = E(\rho)+\lambda\left<\frac{\delta E}{\delta \rho}, (\rho-\nu)\right>+o(\lambda).
\end{equation*}
\end{proof}

\subsection{Landscape Analysis}

In this section we aim to show that the proposed model enjoys a good landscape in the $L_2$ geometry. Specifically, we can always find a descent direction around a point whose loss is strictly larger than 0, which means that all local minimum is a global one.

\begin{theorem}\label{thm:negdirection}
 If $E(\rho)>0$ for some probability distribution $\rho\in\mathcal{P}^2$ which concentrates on one of the nested sets $Q_r$, then there exists a descend direction $v\in \mathcal{P}^2$ s.t.
$$
\left<\frac{\delta E}{\delta \rho} ,(\rho-v)\right> >0
$$
\end{theorem}

\begin{proof}
 
 First we lower bound the gradient respect to the feature map $X_\rho(\cdot,t)$ by the loss function to show that changing feature map can always leads to a lower loss. This is observed by \cite{bartlett2018representing,bartlett2019gradient} where they mean by 

\begin{lemma}\label{lem:gradient} The norm of the solution to the adjoint equation can be bounded by the loss 
$$\|p_\rho(\cdot,t)\|^2_\mu\ge {e^{-(C_1+C_2r)}} E(\rho), \qquad \forall\, t \in [0, 1].$$
\end{lemma}
\begin{proof}
By definition, 
\begin{equation*}
 \|p_\rho(\cdot,1)\|=\|\bigl(\langle w_1,X_\rho(\cdot,1)\rangle-y(\cdot)\bigr)w_1\|=\lvert \langle w_1,X_\rho(\cdot,1)\rangle -y(\cdot)\rvert,
\end{equation*}
which implies that  $\|p_\rho(\cdot,1)\|_\mu^2 =2 E(\rho)$. 

By assumption there exist a constant $C_\rho > 0$ such that
    $$
   \Big|\int \rho(\theta, t)d\theta - \int \rho(\theta, s)d\theta \Big|\leq  \|\rho(\cdot,t-s)-\rho(\cdot,s)\|_{BL}\le C_\rho |t-s|, \ \forall t, s\in[0,1].
    $$
    Integrating the inequality above with respect to $s$ over $[0,1]$, and using the fact that $\int_\theta\int_t \rho(\theta,t) =1 $, one obtains that
    $\int \rho(\theta, t) d\theta \leq 1 + C_\rho \int_0^1 |t-s|ds \leq 1+ \frac{C_\rho}{2}$.
    
    Recall that $p_\rho $ solves the adjoint equation
 \begin{equation}\label{eq:adjointp}
 \dot p_\rho (x,t) = -p_\rho(x,t) \int \nabla_X f(X_\rho(x,t),\theta)\rho(\theta,t)d\theta
 \end{equation}
  where by the assumption on $f$ and the above bound on $\int \rho(\theta, t) d\theta$,  we have for any $x$ 
 $$\|\int \nabla_X f(X_\rho(x,t),\theta)\rho(\theta,t)d\theta\|\leq \sup_{x, \theta} \lvert \nabla_X f(X_{\rho}(x, t), \theta) \rvert  \int_\theta\rho(\theta,t)d\theta\leq(C_1+C_2r).$$
  It then follows from the Gronwall's inequality that
 \begin{equation*}
     \|p_\rho(\cdot,t)\|_\mu\ge e^{-\int_0^1 \sup_x \|\int \nabla_X f(X_\rho(x,t),\theta)\rho(\theta,t)d\theta\| dt}\|p_\rho(\cdot,1)\|_\mu \ge {e^{-(C_1+C_2r)}}E(\rho)^{1/2}.
 \end{equation*}
 The claim of the Lemma then follows by squaring the inequality (and redefining constants $C_1$ and $C_2$).
\end{proof}
Thanks to  the existence and uniqueness of the solution of the ODE model as stated in Theorem \ref{thm:wellpose}, the solution map of the ODE is invertible so that there exists an inverse map $X_{\rho,t}^{-1}$ such that we can construct an inversion function $X_{\rho,t}^{-1}(X_\rho(x,t))=x$. With $X_{\rho,t}^{-1}$, we define 
$
\hat p_\rho(x,t) = p_\rho(X_{\rho,t}^{-1}(x),t)
$.

Since $\rho(\theta, t)$ is a probability density, i.e., $\int \int \rho(\theta,t) d \theta d t = 1$,  there exists $t_\ast \in (0,1)$ such that   $\int_\theta\rho(\theta, t_\ast)d\theta>\frac{1}{2}$. Since $k(x_1,x_2)=f(x_1,x_2)$ is a universal kernel \cite{micchelli2006universal},
for any $g(x)$ satisfying that $\|g\|_{\hat{\mu}} < \infty$ for some probability measure $\hat{\mu}$ and for 
any  fixed $\epsilon > 0$, there exists a probability distribution $\delta \hat\nu\in\mathcal{P}^2(\mathbb{R}^{d_2})$ such that
\begin{equation}\label{eq:bdeps}
    \|g(x)-\int_\theta f(x,\theta)\delta\hat\nu(\theta)d\theta\|_{\hat{\mu}}\le \epsilon,
\end{equation}
In particular, in what follows we consider the function $g(x)$ and the measure $\hat\mu$  given by  
$$g(x) := -\hat p(x,t_\ast)+\frac{1}{\int_\theta \rho(\theta, t_\ast)d\theta}\int_\theta f(x,\theta)\rho(\theta, t_\ast)d\theta \text{ and } \hat{\mu} = \hat \mu_{\rho,t_\ast} := X_\rho(\cdot,t_\ast)_\#\mu.$$ 

The value of $\epsilon$ will be chosen later in the proof. 
Moreover, we also define the perturbed measure 
\begin{equation}\label{eq:deltamu}
    \delta\nu = \left(\delta \hat\mu(\theta)-\frac{\rho(\theta, t_\ast)}{\int_\theta \rho(\theta, t_\ast)d\theta}\right) \phi(t),
\end{equation}
where $\phi(t)$ is a smooth non-negative function integrates to $1$ and compactly supported in the interval $(0, 1)$, so that it is clear that $\delta \nu$ satisfies the regularity assumptions. We will consider the perturbed
probability density $\nu$ defined as 
$$
\nu = \rho + \delta r \delta\nu \text{ for some } \delta r>0.
$$

\begin{lemma} The constructed $\nu$ with $\epsilon$ sufficiently small gives a descent direction of our model with the estimate \begin{equation}
\label{equ:descent}
    \left\langle\frac{\delta E}{\delta \rho},(\nu-\rho)\right\rangle \le - \frac{\delta r}{2} e^{-2(C_1+C_2r)}  E(\rho)<0. 
\end{equation}
\label{lemma:descent}
\end{lemma}

\begin{proof}

 An application of the Gronwall inequality to \eqref{eq:adjointp}  implies that 
 \begin{equation}\label{eq:prodprho}
      p_\rho(x,t_1)p_\rho(x,t_2)\ge e^{-|t_1-t_2|(C_1+C_2r)} \bigl( p_\rho(x,t_1)^2 \vee p_{\rho}(x, t_2)^2 \bigr)
 \end{equation}
 for all $x\in \mathbb{R}^d, 1\geq t_2\geq t_1\geq 0$.
 
 As a result of \eqref{eq:deltamu},
 \begin{align*}
    & \left\langle\frac{\delta E}{\delta \rho},(\nu-\rho)\right\rangle= \mathbb{E}_{x\sim\mu} \left\langle  f(X_{\rho}(x,t),\cdot))p_\rho(x,\cdot),\delta r\delta\nu\right\rangle\\
    & = \delta r \int \mathbb{E}_{x\sim\hat \mu_{\rho,t}} \hat p_\rho(x,t)\int_\theta f( x,\theta)\delta\nu(\theta, t)d\theta \phi(t) dt\\
    & = \delta r \int \mathbb{E}_{x\sim\hat \mu_{\rho,t}} \Bigl[ \hat p_\rho(x,t)\int_\theta f( x,\theta)\delta\hat\nu(\theta)d\theta \Bigr] \phi(t) dt \\
    & \qquad \qquad 
- \delta r \int \mathbb{E}_{x\sim\hat \mu_{\rho,t}}  
\Big [ \hat p_\rho(x,t)  \underbrace {\frac{\int f(x, \theta)\rho(\theta, t_\ast) d\theta}{\int_\theta \rho(\theta, t_\ast)d\theta}}_{= g + \hat p(x, t_\ast)}\Big] dt\\
& = \delta r \int \mathbb{E}_{x\sim\hat \mu_{\rho,t}} \Big[ \hat p_\rho(x,t)\Big(\int_\theta f( x,\theta)\delta\hat\nu(\theta)d\theta - g(x)\Big) \Big] \phi(t) dt \\
& \qquad \qquad 
- \delta r \int \mathbb{E}_{x\sim\hat \mu_{\rho,t}}  
\Big [ \hat p_\rho(x,t) \hat p(x, t_\ast)\Big]\phi(t) dt\\
& =:I_1 + I_2.
 \end{align*}
 The last equation defines $I_1$ and $I_2$ which will be estimated separately below. 
 
 Thanks to \eqref{eq:bdeps}, for $I_1$, we have 
 \begin{equation*}
   \begin{aligned}
     I_1 & \leq \delta r \int \lVert\hat{p}_{\rho}(\cdot, t)\rVert_{\hat{\mu}_{\rho, t}} \bigl\lVert \int_\theta f( x,\theta)\delta\hat\nu(\theta)d\theta - g(x) \bigr\rVert_{\hat{\mu}_{\rho, t}} \phi(t) dt \\
     & = \delta r \int \lVert p_{\rho}(\cdot, t)\rVert_{\mu} \bigl\lVert \int_\theta f( x,\theta)\delta\hat\nu(\theta)d\theta - g(x) \bigr\rVert_{\hat{\mu}_{\rho, t}} \phi(t) dt \\
     & \leq \delta r \int\lVert p_{\rho}(\cdot, t)\rVert_{\mu}\,
     \epsilon  \sup_x \Bigl\lvert \frac{d\hat{\mu}_{\rho, t}}{d \hat{\mu}_{\rho, t_{\ast}}}\Bigr\rvert \phi(t) d t \\
     & = \delta r \int\lVert p_{\rho}(\cdot, t)\rVert_{\mu}\, \epsilon  \sup_x \bigl\lvert J_{\rho}(x; t, t_{\ast}) \bigr\rvert \phi(t) d t,
   \end{aligned}
 \end{equation*}
 where $J_{\rho}(x; t, s)$ is the Jacobian of the flow at time $t$ with respect to time $s$ assuming starting at $x$ at time $0$; which is bounded by the Lipschitz assumption of the $f$. Thus, we have 
 \begin{equation}\label{eq:I1bound}
    I_1 \leq C \epsilon \delta r \int\lVert p_{\rho}(\cdot, t)\rVert_{\mu} \phi(t) d t. 
 \end{equation}
 Thanks to \eqref{eq:prodprho}, one has 
 \begin{equation}\label{eq:I2bound}
  \begin{aligned}
 I_2 & \leq - \delta r \int  e^{-|t-t_\ast|(C_1 +C_2 r)} \|\hat{p}_\rho(\cdot, t)\|_{\hat \mu_{\rho, t}}^2 \phi(t) dt\\ 
 & = - \delta r \int  e^{-|t-t_\ast|(C_1 +C_2 r)} \|p_\rho(\cdot, t)\|_{\mu}^2 \phi(t) dt\\ 
 & \leq - e^{-(C_1 + C_2 r)} \delta r  \int\|p_\rho(\cdot, t)\|_{\mu}^2 \phi(t) dt. 
 \end{aligned}
 \end{equation}
 Combining the above together, and choosing $\epsilon$ sufficiently small that the right-hand-side
 of \eqref{eq:I1bound} is bounded by a half of the right-hand-side of \eqref{eq:I2bound} (note that
 the constants and the integral in the right-hand-side of \eqref{eq:I1bound} and \eqref{eq:I2bound}
 do not depend on $\epsilon$), we arrive at
 \begin{equation*}
 \begin{aligned}
    I_1 + I_2 & \leq - \frac{1}{2} e^{- (C_1 + C_2 r)} \delta r \int \|p_\rho(\cdot, t)\|_{\mu}^2 \phi(t) dt \\
    & \leq - \frac{1}{2} e^{- (C_1 + C_2 r)} \delta r\int e^{-(C_1 + C_2 r)} E(\rho) \phi(t) dt \\
    & = - \delta r \frac{1}{2} e^{-2(C_1 + C_2 r)} E(\rho),
    \end{aligned}
 \end{equation*}
 where the last inequality follows from Lemma~\ref{lem:gradient}. 
\end{proof}

Now we go back to the proof of Theorem~\ref{thm:negdirection}, as Lemma \ref{lemma:descent} illustrates, if the loss $E(\rho)$ is not equal to zero, then we can always find a direction to decrease the loss, this complete the proof. 
\end{proof}

\subsection{Discussion of the Wasserstein gradient flow}

As described in the introduction, we consider each residual block as a particle and trace the evolution of the empirical distribution $\rho_s$ of the particles during the training (here the variable $s$ denotes the training time). While using gradient descent or stochastic gradient descent with small time steps, we move each particle through a velocity field $\{v_s\}_{s\ge 0}$ and the evolution can be expressed by a PDE $\partial_s\rho_s=\text{div}(\rho_sv_s)$, where $\text{div}$ is the
divergence operator. Several recent papers \cite{mei2018mean,chizat2018global,rotskoff2018neural} have shown that when the gradient field is gained from a (stochastic) gradient descent algorithm for training a particle realization of the mean-field model, the PDE is the Wasserstein gradient flow of the objective function. Thus in this section, we consider the gradient flow of the the objective function in the Wasserstein space, given by a McKean--Vlasov type equation~\cite{carrillo2003kinetic,ambrosio2008gradient,jordan1998variational,otto2001geometry,nitanda2017stochastic}
\begin{equation}\label{eq:wasser}
\frac{\partial_{(\theta,t)}\rho}{\partial s}=\text{div}_{(\theta,t)}\left(\rho\nabla_{(\theta,,t)}\frac{\delta E}{\delta \rho}\right).
\end{equation}
We consider the stationary point of such flow, \emph{i.e.}, distribution $\rho$ such that the right hand side is $0$. Our next result shows that such stationary points are global minimum of the loss function under the homogeneous assumption of the residual block and a separation property of the support of the stationary distribution.

\subsection*{Regularity in the Wasserstein Space}

To address the regularity of the Wasserstein gradient flow, following \cite{chizat2018global}, we
first analyze the regularity of $E$ restricted to the set $\{\rho\mid \rho\in\mathcal{P}^2, \rho(Q_r)=1\}$, to make this explicit, we denote the functional $F_r$ as
\begin{align*}
    F_r(\rho)=
    \begin{cases}
             E(\rho), & \text{if } \rho(Q_r)=1;\\
             \infty, &\text{otherwise}.
    \end{cases}
\end{align*}

\begin{theorem*} [Geodesically semiconvex property of $F_{r}$ in Wasserstein geometry] 
 Further assume that
  $f(x,\theta)$ have second order smoothness, \emph{i.e.} $f(x,\theta)$ has a smooth Hessian. Then
  for all $r>0$, $F_r$ is proper and continuous in $W_2$ space on its closed domain, Moreover, for
  $\forall \rho_1,\rho_2\in\mathcal{P}^2$ and an admissible transport plan $\gamma$, denote the
  interpolation plan in Wasserstein space as $\mu_t^\gamma :=((1-t)\rho_1+t\rho_2)_\#\gamma$. There exists a $\lambda>0$ such that the function on the Wasserstein geodesic $t\rightarrow F_r(\mu_t^{\gamma})$ is differentiable with a
  $\lambda C(\gamma)$-Lipschitz derivative. Here $C(\gamma)$ is the transport cost
  $C(\gamma)=\left(\int|y-x|^2d\gamma(x,y)\right)^{1/2}$.
\end{theorem*}

\begin{proof}
To prove the regularity of our objective in the Wasserstein space, we first provide some analysis of the objective function.
\begin{lemma}
  The gradient of the objective function has the following bound, \emph{i.e.}
  \[
  \sup_{\theta\in Q_r} \left\|\frac{\delta E}{\delta \rho}(\theta,t)\right\| = \sup_{\theta\in
    Q_r}\left \|\mathbb{E}_{x\sim\mu} f(X_{\rho}(x,t),\theta))p_\rho(x,t)\right\|\le
  e^{(C_1+C_2r)}\sigma_3(\sigma_2R_1+R_2+C_f^r).
  \]
\end{lemma}

\begin{proof}
First the output of the neural network satisfies 
\[\|X_\rho(x,1)\|\le\|X_\rho(x,0)\|+\|\int_0^1\int_\theta f(X_\rho(x,t),\theta)\rho(\theta,t)d\theta dt\|\le \sigma_2R_1+C_f^r,\] 
thus $\|p_{\rho}(x, 1)\| := \|\frac{\partial E(x; \rho)}{\partial X_{\rho}(x, 1)}\| = \|\bigl( \langle w_1, X_{\rho}(x, 1)\rangle - y(x) \bigr)\|\le \sigma_3(\sigma_2R_1+R_2+C_f^r)$.

At the same time, for the adjoint process $p_\rho(x,t)$ satisfying  the adjoint equation, using Gronwall inequality we have, similarly to the proof of Lemma~\ref{lem:gradient}
\begin{equation}\label{eq:boundprho}
  \|p_\rho(\cdot,t)\|\le e^{\int_0^1\|\int_\theta\nabla_X
    f(X_\rho(x,t),\theta)\rho(\theta,t)d\theta\|dt}\|p_\rho(\cdot,1)\| \le
  e^{(C_1+C_2r)}\sigma_3(\sigma_2R_1+R_2+C_f^r).
\end{equation}
The conclusion then follows as $f$ is bounded on the compact space. 
\end{proof}

\begin{lemma}
  The gradient of the objective function with respect to the feature $X_\rho(x,t)$ is Lipschitz
  in $\mathcal{P}^2$, \emph{i.e.}, there exists a constant $L_{g_1}$ satisfies
  \[
  \sup_{\rho_1\not=\rho_2}\sup_{s\in(0,1)} \frac{\|p_{\rho_1}(x,s)-p_{\rho_2}(x,s)\|}{\|\rho_1-\rho_2\|}\le L_{g_1}.
  \]
  Furthermore, the Frechet derivative $\frac{\delta p_\rho}{\delta \rho}$ exists.
\end{lemma}

\begin{proof}
  As proved in Theorem 1, $\|X_{\rho_1}(x,1)-X_{\rho_2}(x,1)\|\le \hat C e^{\hat
    C}d_{W}(\rho_1,\rho_2)\le \frac{\hat C e^{\hat C}}{R_r^2} \|\rho_1-\rho_2\|$, which leads to
  $\|p_{\rho_1}(x,1)-p_{\rho_2}(x,1)\|=|(\left<w_1,X_{\rho_1}(x_1,1)\right>-y(x))|-|(\left<w_1,X_{\rho_2}(x_1,1)\right>-y(x))|\le
  \hat C e^{\hat C}d_{W}(\rho_1,\rho_2)\le \frac{\hat C e^{\hat C}}{R_r^2} \|\rho_1-\rho_2\|$. To propagate the estimates to $t\leq 1$, we control 
  \begin{align*}
    \left\|\dot p_{\rho_1}(x,s)-\dot p_{\rho_2}(x,s)\right\| &= \Biggl\|\left(\int_\theta \nabla_X  f(X_{\rho_1}(x,s),\theta)\rho_1(\theta,s)d\theta\right) p_{\rho_1}(x,s) \\
    & \qquad \qquad -\left(\int_\theta \nabla_X f(X_{\rho_2}(x,s),\theta)\rho_2(\theta,s)d\theta\right)p_{\rho_2}(x,s)\Biggr\| \\
    &\le \left\|\left(\int_\theta\nabla_X f(X_{\rho_1}(x,s),\theta)\rho_1(x,s)d\theta\right)( p_{\rho_1}(x,s)- p_{\rho_2}(x,s))\right\|\\
    &\qquad \qquad +\left\|\left(\int_\theta \nabla_X f(X_{\rho_2}(x,s),\theta)(\rho_2(x,s)-\rho_1(x,s))d\theta\right)p_{\rho_2}(x,s)\right\|\\
    &\le (C_1+C_2r) \left(\int\rho_1(\theta,s)d\theta \right)\| p_{\rho_1}(x,s)- p_{\rho_2}(x,s)\|\\&\qquad \qquad +(C_1+C_2r)\|p_{\rho_2}(x,s)\| \Biggl(\int_\theta\left(\rho_1(\theta,s)-\rho_2(\theta,s)\right)^2d\theta \Biggr)^{1/2}\\
    &\stackrel{\eqref{eq:boundprho}}{\le} (C_1+C_2r) \left(\int\rho_1(\theta,s)d\theta \right)\| p_{\rho_1}(x,s)- p_{\rho_2}(x,s)\|\\&\qquad \qquad +(C_1+C_2r)e^{(C_1+C_2r)}\sigma_3(\sigma_2R_1+R_2+C_f^r) \\
    & \qquad\qquad\qquad\times\Biggl(\int_\theta\left(\rho_1(\theta,s)-\rho_2(\theta,s)\right)^2d\theta \Biggr)^{1/2}.
  \end{align*}
  
  Introduce the short hand $M: = (C_1+C_2r)e^{(C_1+C_2r)}\sigma_3(\sigma_2R_1+R_2+C_f^r)$ and applying the Gronwall inequality, we obtain 
  \begin{align*}
    \left\| p_{\rho_1}(x,s)- p_{\rho_2}(x,s)\right\| & \le \frac{\hat C e^{\hat C+(C_1+C_2r)\int_0^1 \int\rho_1(\theta,s)d\theta ds}}{R_r^2} \|\rho_1-\rho_2\|  \\&\qquad+ \int_0^1 M e^{(C_1+C_2r)\int_t^1 \left(\int\rho_1(\theta,s)d\theta \right)ds}  \Biggl(\int_\theta\left(\rho_1(\theta,s)-\rho_2(\theta,s)\right)^2d\theta \Biggr)^{1/2} dt\\
    & \leq  \frac{\hat C e^{\hat C+(C_1+C_2r)}}{R_r^2} \|\rho_1-\rho_2\|  \\&\qquad+  M e^{(C_1+C_2r)\int_0^1 \int\rho_1(\theta,s)d\theta ds}  \int_0^1 \Biggl(\int_\theta\left(\rho_1(\theta,s)-\rho_2(\theta,s)\right)^2d\theta \Biggr)^{1/2} dt\\
    &\le \biggl(\frac{\hat C e^{\hat C+(C_1+C_2r)}}{R_r^2}+M e^{(C_1+C_2r)} \biggr) \|\rho_1-\rho_2\|,
  \end{align*}
  where last inequality follows from Jensen's inequality
  \begin{equation*}
      \int_0^1 \Biggl(\int_\theta\left(\rho_1(\theta,s)-\rho_2(\theta,s)\right)^2d\theta \Biggr)^{1/2} dt \leq \Biggl(\int_0^1 \int_\theta\left(\rho_1(\theta,s)-\rho_2(\theta,s)\right)^2d\theta dt \Biggr)^{1/2} = \lVert \rho_1 - \rho_2 \rVert. \end{equation*}
  
  The existence of the Frechet derivative follows from the smoothness of the activation function, in particular the assumption that the Hessian is bounded. 
 \end{proof}

Now we show the continuity of the objective function in the Wasserstein space. By denoting
$h(\tau)=F_r(\mu_\tau^\gamma)$
\begin{align}
h'(\tau) &= \frac{d}{d\tau}F_r(\mu_\tau^\gamma) \nonumber\\
      &= \left\langle \frac{\delta E}{\delta \rho}[\mu_\tau^\gamma], \frac{d}{d\tau}\mu^\gamma_\tau \right\rangle \nonumber\\
      &= \int d\frac{\delta E}{\delta \rho}[\mu_\tau^\gamma]((1-\tau)(\theta_1, t_1)+\tau(\theta_2, t_2))((\theta_1,t_1)-(\theta_2,t_2))d\gamma((\theta_1,t_1),(\theta_2,t_2)).
\end{align}
For any $\tau_1,\tau_2\in[0,1]$, we have $h'(\tau_1)-h'(\tau_2)=I+J$ with
\begin{align}
I &= \int d\frac{\delta E}{\delta \rho}[\mu_{\tau_1}^\gamma]((1-\tau_1)(\theta_1, t_1)  +\tau_1(\theta_2, t_2))((\theta_1,t_1)-(\theta_2,t_2))d\gamma((\theta_1,t_1),(\theta_2,t_2)) \nonumber\\
  &\quad- \int d\frac{\delta E}{\delta \rho}[\mu_{\tau_2}^\gamma]((1-\tau_1)(\theta_1, t_1)+\tau_1(\theta_2, t_2))((\theta_1,t_1)-(\theta_2,t_2))d\gamma((\theta_1,t_1),(\theta_2,t_2)), \\
J &= \int d\frac{\delta E}{\delta \rho}[\mu_{\tau_2}^\gamma]((1-\tau_1)(\theta_1, t_1)  +\tau_1(\theta_2, t_2))((\theta_1,t_1)-(\theta_2,t_2))d\gamma((\theta_1,t_1),(\theta_2,t_2)) \nonumber\\
  &\quad- \int d\frac{\delta E}{\delta \rho}[\mu_{\tau_2}^\gamma]((1-\tau_2)(\theta_1, t_1)+\tau_2(\theta_2, t_2))((\theta_1,t_1)-(\theta_2,t_2))d\gamma((\theta_1,t_1),(\theta_2,t_2)).
\end{align}
For $I$, we have
\begin{align}
|I| &\leq L_{g_1}\cdot2r\|\mu^\gamma_{\tau_1}-\mu^\gamma_{\tau_2}\| \nonumber\\
    &\leq 2rL_{g_1}C_2(\gamma)|\tau_1-\tau_2|.
\end{align}
Similarly, for $J$ we have
\begin{align}
|J| &\leq L_{g_1}|\tau_1-\tau_2|\int ((\theta_1,t_1)-(\theta_2,t_2))^2d\gamma \nonumber\\
    &= L_{g_1}C_2^2(\gamma)|\tau_1-\tau_2|. 
\end{align}
Finally, combining the estimates for $I$ and $J$ shows that $h'(\tau)$ is Lipschitz continuous. \end{proof}

With the proved regularity, the short time well-posedness of Wasserstein gradient flow is a corollary of Theorem 11.2.1 of \cite{ambrosio2008gradient}.

\begin{corollary}   There exists a $T_{\max}$ such that there exists a unique solution $\{\rho_s\}_{s\in[0,T_{\max}]}$ to the Wasserstein gradient flow $
  \frac{\partial_{(\theta,t)}\rho}{\partial
    s}=\text{div}_{(\rho,t)}(\rho\nabla_{(\rho,t)}\frac{\delta E}{\delta \rho}) $ starting from any $\mu_0 \in \mathcal{P}_2$ concentrated on $Q_{r}$.
\end{corollary}

\subsection*{Convergence Results For The Wasserstein Gradient Flow}

We move on to prove that the stationary point of the Wasserstein gradient flow achieves the global optimum with a
support related assumption. Following \cite{chizat2018global}, we introduce
an assumption of the homogeneity of the activation function which is a central requirement for our
global convergence results. 

\paragraph{Homogeneity.} A function $f$ between vector spaces is {\em positively $p$-homogeneous} when for all
$\lambda>0$ and argument $x$, $f(\lambda x)=\lambda^pf(x)$.  We assume that the functions $f(X,\theta)$ that constitute the residual block obtained through the
lifting share the property of being positively $p$-homogeneous ($p>0$) in the variable $\theta$. As \cite{chizat2018global}
remarked the ReLU function is a 1-homogeneity function which leads to the 2-homogeneity respect to
$\theta$ of $f(X,\theta)$ when the residual block is implemented via a two-layer neural network.

\begin{theorem}
  When the residual block $f(X,\theta)$ is positively $p$-homogeneous respective to $\theta$. Let
  $(\rho_s)_{s\ge0}$ be the solution of the the Wasserstein gradient $
  \frac{\partial_{(\theta,t)}\rho}{\partial
    s}=\text{div}_{(\rho,t)}(\rho\nabla_{(\rho,t)}\frac{\delta E}{\delta \rho}) $ of our mean-field
  model (\ref{con:model}). Consider a stationary solution to the gradient flow $\rho_\infty$ which 
  concentrates in one of the nested sets $Q_r$ and separates the spheres $r_a \S^{d-1}\times[0,1]$ and $r_b
  \S^{d-1}\times[0,1]$. Then $\rho_{\infty}$ is a global minimum satisfies $E(\rho_\infty)=0$.
\end{theorem}

\begin{proof}

First we use the conclusion of \cite{nitanda2017stochastic} which characterize the condition of the stationary points in the Wasserstein space, which concludes that the steady state $\rho_{\infty}$ of the Wasserstein gradient flow 
\begin{equation*}
\frac{\partial_{(\theta,t)}\rho}{\partial s}=\text{div}_{(\rho,t)}(\rho\nabla_{(\rho,t)}\frac{\delta
  E}{\delta \rho})
  \end{equation*}
  must satisfy $\nabla_{(\theta,t)}\frac{\delta E}{\delta \rho}\vert_{\rho_{\infty}}=0,\rho_{\infty}\text{-a.e.}$

We will use the homogeneity of the activation function and the separation property of the
support of $\rho_\infty$ to further prove that $\nabla_{(\theta,t)}\frac{\delta E}{\delta
  \rho}|_{\rho=\rho_\infty}=0,\text{a.e.}$ (i.e., it also vanishes outside the support of $\rho_{\infty}$, which might not be the full parameter space).
  
  Due to the separation assumption of the
support of the distribution, for any $(\theta,t)\in\mathbb{R}^{d_1\times d_1}\times [0,1]$, there
exists $r>0$ such that $(r\theta,t)\in \text{supp}(\rho_\infty)$. Due to the homogeneity assumption, we have 
\begin{equation*}
    \frac{\delta E}{\delta
  \rho}(r\theta,t) = \mathbb{E}_{x\sim\mu}
f(X_{\rho}(x,t),r\theta))p_\rho(x,t)=r^p\mathbb{E}_{x\sim\mu}
f(X_{\rho}(x,t),\theta))p_\rho(x,t)=r^p\frac{\delta E}{\delta \rho}(r\theta,t),
\end{equation*}
which leads to
$\nabla_{(\theta,t)}\frac{\delta E}{\delta \rho}(r\theta,t)=r^p\nabla_{(\theta,t)}\frac{\delta
  E}{\delta \rho}(\theta,t)$. Thus, since $\nabla_{(\theta,t)}\frac{\delta E}{\delta
  \rho}|_{\rho=\rho_\infty}=0,\rho_\infty \text{-a.e.}$, we know that
$\nabla_{(\theta,t)}\frac{\delta E}{\delta \rho}|_{\rho=\rho_\infty}=0$, a.e. This further implies that the differential is a constant $\frac{\delta
  E}{\delta \rho}|_{\rho=\rho_\infty} \equiv c$. 

If $E(\rho_\infty)\not=0$, according to Theorem 3, there exists another distribution $\nu\in\mathcal{P}^2$ s.t.
\[
\left\langle\frac{\delta E}{\delta \rho}|_{\rho=\rho_\infty},(\rho-\nu)\right\rangle > 0.
\]
However $\bigl\langle\frac{\delta E}{\delta \rho}|_{\rho=\rho_\infty},(\rho-\nu)\bigr\rangle=c\left(\int
\rho(\theta,t)d\theta dt-\int \nu(\theta,t)d\theta dt \right)=0$ due to the normalization of the probability measure. This leads to a contradiction. Thus the stationary solution measure must
satisfy $E(\rho_\infty)=0$, which means that it is a global optimum.
\end{proof}
\section{Deep ResNet as Numerical Scheme}

In this section, following \cite{bengio2006convex,lu2017beyond}, we aim to design scalable deep learning algorithms via the discretization of the continuous model. We use a set of particles to approximate the the distribution~\cite{nitanda2017stochastic,ba2019towards,liu2016stein} and Euler scheme to numerical solve the ODE model which leads to a simple Residual Network~\cite{lu2017beyond}.

To simulate the Wasserstein gradient flow \eqref{eq:wasser} via a stochastic gradient descent algorithm, we use a particle representation of the distribution $\rho(x,t)$, commonly used in the literature, see e.g.,~\cite{liu2016stein,nitanda2017stochastic,rotskoff2019neuron,mei2018mean,chizat2018global}. In the two-layer neural network, the particle realization becomes the standard training procedure of using (stochastic) gradient descent. Our aim is to extend this approach to deep residual networks, starting from the continuum mean-field model presented above. Since $\rho$ characterizes the distribution of the pairs $(\theta,t)$, each particle in our representation would carry the parameter $\theta$, together with information on the activation time period of the particle.  Therefore, also different from the usual standard ResNet, we also need to allow the particle to move in the gradient direction corresponding to $t$. We may consider using a parametrization of $\rho$ with $n$ particles as
\begin{equation*}
    \rho_n(\theta, t) =\sum_{i=1}^n \delta_{\theta_i}(\theta) \mathds{1}_{[\tau_{i},\tau_{i}']}(t). 
\end{equation*}
The characteristic function $\mathds{1}_{[\tau_{i},\tau_{i}']}$ can be viewed as a relaxation of the Dirac delta mass $\delta_{\tau_{i}}(t)$. However, this parametrization comes with a difficulty in practice, namely, the intervals $[t_{i},t_{i}']$ may overlap significantly with each other, and in the worst case, though unlikely, all the time intervals of the $n$ particles coincide, which leads to heavy computational cost in the training process.

Therefore, for practical implementation, we constrain that every time instance $t$ is just contained in the time interval of a single particle. We realize this by adding a constraint $\tau_i'=\tau_{i+1}$ between consecutive intervals. More precisely, given a set of parameters $(\s^i,\tau^i)$, we first sort them according to $\tau^i$ values. Assuming $\tau^i$ are ordered, we define the architecture as 
\begin{align}
& X^{\ell+1} = X^{\ell} + (\tau^\ell-\tau^{\ell-1}) \sigma(\s^\ell X^{\ell}),\quad 0\le\ell<n; \\
& X^0 = x. 
\end{align}
Both $\s$ and $\tau$ parameters can be trained with SGD and $n$ is the depth of the network. The order of $\tau$ may change during the training (thus to make each particle indistinguishable to guarantee  the mean-field behavior), thus after every update, we {\em sort} the $\tau_i$ to get the new order of the residual blocks. The algorithm is listed in Algorithm \ref{alg:mf}. The new algorithm only introduces $n$ parameters, as $n$ is the depth which is around $100$ in practice, thus the number of extra parameters  is negligible comparing to the $1$M+ parameter number typically used in usual ResNet architectures. The sorting of $\{\tau_i\}_{i=1}^n$ also induces negligible cost per step. 

We also remark that the flexibility of $\tau^{\ell}$ can be also viewed as an adaptive time marching scheme of the ODE model for $x$, as $\tau^{\ell}-\tau^{\ell-1}$ can be understood as the time step in the Euler discretization. Since the parameters $\{\tau^{\ell}\}$ are learned from data, as a by-product, our scheme also naturally yields a data-adaptive discretization scheme.

\begin{algorithm}[H]
\begin{algorithmic}
\State \textbf{Given}: A collection of residual blocks $(\theta_i,\tau_i)_{i=1}^n$\;
\State \textbf{While} {training} \textbf{do}\\
 \indent \indent Sort $(\theta_i,\tau_i)$ based on $\tau_i$ to be $(\theta^i,\tau^i)$ where $\tau^0\le\cdots\le\tau^n$.\\
\indent \indent Define the ResNet as
$X^{\ell+1} = X^{\ell} + (\tau^\ell-\tau^{\ell-1}) \sigma(\s^\ell X^{\ell})$ for $0\le\ell<n$.\\
\indent \indent Use gradient descent to update both $\theta^i$ and $\tau^i$.
\State \textbf{End while}
\end{algorithmic}
 \caption{Training Of Mean-Field Deep Residual Network}
 \label{alg:mf}
\end{algorithm}

As the number of particles $n$ becomes large, the expected time evolution of $\rho_n$ should be close to the gradient flow \eqref{eq:wasser}. The rigorous proof of this is however non-trivial, which will be left for future works.

\section{Experiment}

In this section, we aim to show that our algorithm is not only designed from theoretical consideration but also realizable on practical datasets and network structures. We implement our algorithm for ResNet/ResNeXt on CIFAR 10/100 datasets and demonstrate that our ``mean-field training'' method consistently outperforms the vanilla stochastic gradient descent.

\textbf{Implementation Details.}

On CIFAR, we follow the simple data augmentation method in \cite{he2016deep,he2016identity} for training: 4 pixels are padded on each side, and a 32$\times$32 crop is randomly sampled from the padded image or its horizontal flip. For testing, we only evaluate the single view of the original 32$\times$32 image.

\begin{table}[hbp!]
\begin{center}
\begin{tabular}{l||c|c|c}
\hline
         & Vanilla & mean-field &Dataset \\ \hline\hline
ResNet20 &      8.75           &   8.19     &CIFAR10             \\ \hline
ResNet32 &      7.51           & 7.15     &CIFAR10          \\ \hline
ResNet44 &     7.17            &   6.91    &CIFAR10             \\ \hline
ResNet56 &    6.97             & 6.72       &CIFAR10                \\ \hline
ResNet110& 6.37& 6.10&CIFAR10\\\hline
ResNet164&5.46&5.19&CIFAR10\\\hline
ResNeXt29(8$\times$64d)&17.92&17.53&CIFAR100\\\hline
ResNeXt29(16$\times$64d)& 17.65&16.81&CIFAR100\\\hline
\end{tabular}
\end{center}
\label{table:CIFAR}
\caption{Comparison of the stochastic gradient descent and mean-field training (Algorithm 1.) of ResNet On CIFAR Dataset. Results indicate that our method our performs the Vanilla SGD consistently.}
\end{table}

For the experiments of ResNet on CIFAR, we adopt the original design of the residual block in \cite{he2016deep}, i.e. using a small two-layer neural network as the residual block, whose layered structure is bn-relu-conv-bn-relu-conv.
We start our networks with a single $3\times 3$ conv layer, followed by 3 residual blocks, a global average pooling, and a fully-connected classifier. Parameters are initialized following the method introduced by \cite{He2015Delving}. Mini-batch SGD is used to optimize the parameters with a batch size of 128.
During training, we apply a weight decay of 0.0001 for ResNet and 0.0005 for ResNeXt, and a momentum of 0.9.

For ResNet on CIFAR10 (CIFAR100), we start with the learning rate of 0.1, divide it by 10 at 80 (150) and 120 (225) epochs and terminate the training at 160 (300) epochs. For ResNeXt on CIFAR100, we start with the learning rate of 0.1 and divide it by 10 at 150 and 225 epochs, and terminate the training at 300 epochs. We would like to mention that here the ResNeXt is a preact version which is different from the original \cite{xie2017aggregated}. This difference leads to a small performance drop on the final result. For each model and dataset, we report the average test accuracy over 3 runs in Table \ref{table:CIFAR}.

\section{Discussion and Conclusion}

\subsection{Conclusion}
To better understand the reason that stochastic gradient descent can optimize the complicated landscape. Our work directly consider an infinitely deep residual network. We proposed a new continuous model of deep ResNets and established an asymptotic global optimality property by bounding the difference between the gradient of the deep residual network and an associated two-layer network. Our analysis can be considered as a theoretical characterization of the observation that a deep residual network looks like a shallow model ensemble~\cite{veit2016residual} by utilizing ODE and control theory. Based on the new continuous model, we consider the original residual network as an approximation of the continuous model and proposed a new training method. The new method involves a step of sorting residual blocks, which introduces essentially no extra computational effort but results in better empirical results.

\subsection{Discussion and Future Work}

Our work gives qualitative analysis of the loss landscape of a deep residual network and shows that its gradient differs from the gradient of a two-layer neural network by at most a bounded factor when the loss is at the same level. This indicates that the deep residual network's landscape may not be much more complicate than a two-layer network, which inspires us to formulate a mean-field analysis framework for deep residual network and suggests a possible framework for the optimization of the deep networks beyond the kernel regime. \cite{yun2019deep} has shown that deep residual network may not be better than a linear model in terms of optimization, but our work suggests that this is caused by the lack of overparameterization. In the highly overparameterization regime, the landscape of deep ResNet can still be nice. Based on the initiation and framework proposed in our paper, there are several interesting directions related to understanding and improving the residual networks. 

Firstly, to ensure the full support assumption, we can consider extending the neural birth-death~\cite{rotskoff2019neuron,chizat2019sparse} to deep ResNets. Neural birth-death dynamics considers the gradient flow in the Wasserstein-Fisher-Rao space\cite{chizat2018interpolating} rather than the Wasserstein space and ensures convergence. 

Secondly, as shown in the derivation in Section \ref{ensemble}, the two-layer network approximation is just the lowest order approximation to the deep residual network and it is interesting to explore the higher order terms.

\section*{Acknowledgments}
We thank the hospitality of the American Institute of Mathematics (AIM) for the workshop ``Deep learning and partial differential equation'' in October 2019, which led to this collaborative effort.
Yiping Lu also thanks Denny Wu and Xuechen Li for helpful comments and feedback. The work of Jianfeng Lu is supported in part by the National Science Foundation via grants DMS-1454939 and CCF-1934964 (Duke TRIPODS).
The work of Lexing Ying is supported in part by the National Science Foundation via grant DMS-1818449.

\bibliographystyle{siamplain}
\bibliography{mybibtex}

\end{document}